\documentclass[11pt]{article}
\usepackage[margin=0.8in]{geometry}
\usepackage{amsmath,amssymb,amsthm,booktabs,hyperref,microtype}
\usepackage[T1]{fontenc}
\usepackage{lmodern}
\hypersetup{colorlinks=true,linkcolor=black,urlcolor=blue,citecolor=black}
\newtheorem{theorem}{Theorem}
\newtheorem{corollary}{Corollary}
\newcommand{\TV}{\operatorname{TV}}
\newcommand{\E}{\mathbb{E}}
\title{A Bethe-Ansatz Source of Early Token Correlations\\
\large A marginal-preserving construction for parallel proposals}
\author{Abraham Chachamovits\\ENTRUST AI}
\date{14 October 2026}
\begin{document}
\maketitle

\begin{abstract}
Parallel token proposals need a joint law, not merely accurate token marginals. We construct coordinate-Bethe-ansatz laws from the periodic XXZ chain and compare them with their maximum-entropy pairwise projections at identical one- and two-site moments. In an eight-site, three-excitation sector, symmetry makes the full probability-law difference one-dimensional, determined by a contiguous three-site moment $\delta$. Along the Bethe branch, this gap changes sign and vanishes exactly at $\Delta=(3+\sqrt3)/2$, where the eigenstate law becomes pairwise. A nonlinear two-token readout preserves separate token marginals yet exposes a distance $2|\delta|$; a direct three-token readout preserves \emph{all} one- and two-token marginals yet exposes $4|\delta|$. At $\Delta=1.5$ these deterministic binary distances are $0.02654573$ and $0.05309145$. We give an exact target-alignment formula and prove a strict speculative accepted-prefix advantage throughout an unrestricted open neighborhood of a constructed target. These are finite-law existence results, not measured language-model or wall-clock gains.
\end{abstract}

\section{Question and construction}

Parallel token prediction can produce dependent proposals in one model call \cite{ptp}. The question here is narrower: can an integrable spin state's correlation structure generate a token dependence that is invisible to a control matching its one- and two-site spin statistics, while preserving lower-order proposed-token marginals? We also ask how the detectable difference depends on the XXZ anisotropy, so the Bethe scattering structure contributes more than the provenance of one example. The coordinate Bethe ansatz originates with Bethe's spin-chain analysis \cite{bethe}. Our construction does not require language itself to obey a spin-chain Hamiltonian.

Let $Z=(Z_0,\ldots,Z_{L-1})\in\{0,1\}^{L}$ denote occupation variables in the $M$-excitation sector $\Omega_{L,M}=\{z:\sum_jz_j=M\}$. A normalized real eigenvector $\psi$ of the periodic spin-$\frac12$ XXZ Hamiltonian induces an ordinary classical probability law
\begin{equation}\label{eq:born}
 P(z)=|\psi(z)|^2,\qquad z\in\Omega_{L,M}.
\end{equation}
The comparator $Q$ is the maximum-entropy law on the same sector with the same $\E Z_i$ and $\E Z_iZ_j$ for all $i<j$. This is a strong pairwise control: any difference between $P$ and $Q$ begins beyond their matched one- and two-site moments. Pairwise maximum-entropy models have a standard statistical-physics interpretation \cite{roudi}.

Choose distinct sites $i,j,k$ and define two latent readouts
\begin{equation}\label{eq:readout}
 A=Z_i,\qquad B=Z_jZ_k.
\end{equation}
$A$ and $B$ each have the same marginal under $P$ and $Q$: the former is a one-site event, the latter a two-site event. But their joint event $\{A=B=1\}$ measures the three-site moment. Write
\begin{equation}\label{eq:delta}
 \delta=\E_P[Z_iZ_jZ_k]-\E_Q[Z_iZ_jZ_k].
\end{equation}
If $\delta\ne0$, the two arms have different \emph{two-token} dependence even though their latent comparator matches every pair of original sites. This order reduction comes from the nonlinear readout $B=Z_jZ_k$; it is not a claim that the pairwise spin moments differ.

\section{Marginal-preserving token law}

For proposed token positions $t=1,2$, let $K_t^0$ and $K_t^1$ be probability distributions on any finite vocabulary $\mathcal V_t$. Under either arm, sample $(A,B)$ jointly, then sample $T_1\sim K_1^A$ and $T_2\sim K_2^B$ conditionally independently. The channels $K_t^b$ are identical for the two arms. Let $\mathsf P$ and $\mathsf Q$ denote the resulting token-pair laws and put $d_t=\TV(K_t^0,K_t^1)$.

\begin{theorem}[Early-token correlation transfer]\label{thm:transfer}
If $P$ and $Q$ have identical one- and two-site moments, the individual token laws of $T_1$ and $T_2$ are identical across arms. Moreover,
\begin{equation}\label{eq:factor}
 \mathsf P-\mathsf Q
 =\delta\,(K_1^0-K_1^1)\otimes(K_2^0-K_2^1),
 \qquad
 \TV(\mathsf P,\mathsf Q)=2|\delta|d_1d_2.
\end{equation}
Thus a nonzero three-site moment gap produces a nonzero two-token joint gap whenever both channels distinguish their binary inputs.
\end{theorem}

\begin{proof}
Let $r_1=\Pr(A=1)$ and $r_2=\Pr(B=1)$, common to both arms. Their $2\times2$ latent probability tables differ only through $\Pr(A=1,B=1)$, which differs by $\delta$. Fixing both row and column sums forces the signed table difference, in $(0,0),(0,1),(1,0),(1,1)$ order, to be $(\delta,-\delta,-\delta,\delta)$. Pushing this difference through the common channels gives the tensor product in \eqref{eq:factor}. The first marginal is $(1-r_1)K_1^0+r_1K_1^1$, and the second is $(1-r_2)K_2^0+r_2K_2^1$, both independent of the arm. Finally, the $\ell_1$ norm of a tensor product factors, so its half-norm is $\tfrac12|\delta|(2d_1)(2d_2)=2|\delta|d_1d_2$.
\end{proof}

There is also a direct readout that requires no nonlinear site map: take the ordered binary triple $(Z_i,Z_j,Z_k)$ itself. Denote its laws under $P,Q$ by $P^{(3)},Q^{(3)}$. For each token position $t=1,2,3$, optionally apply the same arm-independent conditional emission channel $K_t^{b}$ to binary input $b$; write $\mathsf P^{(3)},\mathsf Q^{(3)}$ for the resulting laws and $d_t=\TV(K_t^0,K_t^1)$.

\begin{theorem}[Direct three-token correlation transfer]\label{thm:triple}
Under the same moment-matching assumption, $P^{(3)}$ and $Q^{(3)}$ agree on every one- and two-token marginal. For every $b=(b_1,b_2,b_3)\in\{0,1\}^3$,
\begin{equation}\label{eq:triplecells}
 P^{(3)}(b)-Q^{(3)}(b)=\delta\prod_{t=1}^{3}(2b_t-1),
 \qquad \TV(P^{(3)},Q^{(3)})=4|\delta|.
\end{equation}
The same lower-order marginals match after common emission channels, and
\begin{equation}\label{eq:triplefactor}
 \mathsf P^{(3)}-\mathsf Q^{(3)}
 =\delta\bigotimes_{t=1}^{3}(K_t^1-K_t^0),\qquad
 \TV(\mathsf P^{(3)},\mathsf Q^{(3)})=4|\delta|d_1d_2d_3.
\end{equation}
\end{theorem}

\begin{proof}
The indicator of cell $b$ is $\prod_{t=1}^3[b_tZ_{i_t}+(1-b_t)(1-Z_{i_t})]$ with $(i_1,i_2,i_3)=(i,j,k)$. In its expansion, all terms of degree at most two have identical expectations under $P,Q$. The remaining coefficient of $Z_iZ_jZ_k$ is $\prod_t(2b_t-1)$, proving \eqref{eq:triplecells}. Each of eight absolute cell differences is $|\delta|$. Summing the signed cells against the product channels gives \eqref{eq:triplefactor}; applying any common channel to an equal lower-order marginal preserves its equality. Finally $\|K_t^1-K_t^0\|_1=2d_t$, giving half the product norm $4|\delta|\prod_td_t$.
\end{proof}

Both results allow prescribed full-vocabulary token marginals. For each readout bit with common success probability $r_t$, given a desired marginal $\mu_t$ with at least two positive entries, choose a nonzero bounded function $s_t$ with $\sum_x\mu_t(x)s_t(x)=0$. For any sufficiently small $\varepsilon>0$, define
\begin{equation}\label{eq:channel}
 K_t^0(x)=\mu_t(x)[1-\varepsilon r_ts_t(x)],\qquad
 K_t^1(x)=\mu_t(x)[1+\varepsilon(1-r_t)s_t(x)].
\end{equation}
Both kernels normalize, remain nonnegative for sufficiently small $\varepsilon$, and satisfy $(1-r_t)K_t^0+r_tK_t^1=\mu_t$ exactly. Nonconstant $s_t$ gives $d_t>0$. Small amplitude is a sufficient construction, not a general limit: if a subset $E_t\subseteq\mathcal V_t$ has $\mu_t(E_t)=r_t$, the disjoint-support channels $K_t^1=\mu_t(\cdot\mid E_t)$ and $K_t^0=\mu_t(\cdot\mid E_t^c)$ preserve $\mu_t$ and attain $d_t=1$. Thus any prescribed nondegenerate marginals admit an observable correlation, and compatible marginals can retain the entire latent gap. These are classical emission mechanisms.

\section{Exact verification and target alignment}

The token-law distinction has a precise acceptance consequence for a defined target family. Let a two-token autoregressive target have joint law $R$. The two proposal laws $\mathsf P,\mathsf Q$ share first marginal $m$; let the target first marginal be $r$, which may differ from $m$. Use ordinary exact speculative acceptance $\min\{1,R_t(x\mid h)/q_t(x\mid h)\}$ and the standard residual correction on rejection \cite{leviathan}. Let $L(U,R)$ be the expected number of \emph{accepted proposed tokens} in a two-token block. Conditional distributions at zero-mass prefixes can be specified arbitrarily where their prefactor below is zero.

\begin{theorem}[Exact target-alignment identity]\label{thm:alignment}
For any such target $R$, without a restriction on its first marginal,
\begin{equation}\label{eq:alignment}
 L(\mathsf P,R)-L(\mathsf Q,R)=\sum_x\min\{m(x),r(x)\}
 \left[\TV\bigl(\mathsf Q(\cdot\mid x),R(\cdot\mid x)\bigr)
 -\TV\bigl(\mathsf P(\cdot\mid x),R(\cdot\mid x)\bigr)\right].
\end{equation}
\end{theorem}

\begin{proof}
For either proposal the joint probability of proposing and accepting first token $x$ is $a_x=\min\{m(x),r(x)\}$. Given that accepted prefix, the chance of accepting the second proposal from $U$ is $1-\TV(U(\cdot\mid x),R(\cdot\mid x))$. The first-token contributions cancel between arms, yielding \eqref{eq:alignment}.
\end{proof}

\begin{corollary}[Constructed target and robustness]\label{cor:accept}
Set $R=\mathsf P$. The expected number of accepted proposed tokens out of two is $2$ for the XXZ-derived arm and $2-D$ for the matched pairwise arm, where
\begin{equation}\label{eq:gain}
 D=\TV(\mathsf P,\mathsf Q)=2|\delta|d_1d_2.
\end{equation}
For any target $R$ with $\TV(R,\mathsf P)\le\eta$, even when its first marginal differs, the advantage is at least $D-5\eta$, and hence is positive on the unrestricted open ball $\TV(R,\mathsf P)<D/5$. If $R$ has the same first-token marginal, the sharper bound $D-2\eta$ applies, positive for $\eta<D/2$. Both are sufficient, not necessarily optimal, radii.
\end{corollary}

\begin{proof}
At the center $R=\mathsf P$, both first-token laws agree with the target, so $L(U,R)=2-\TV(U,R)$; this gives the stated lengths. For the unrestricted bound put $a_x=\min(m_x,r_x)$, $d_x=\TV(\mathsf P(\cdot\mid x),\mathsf Q(\cdot\mid x))$, $e_x=\TV(\mathsf P(\cdot\mid x),R(\cdot\mid x))$, and $\eta=\TV(\mathsf P,R)$. The bracket in \eqref{eq:alignment} is at least $d_x-2e_x$ by two triangle inequalities. As $D=\sum_xm_xd_x$,
\begin{equation}\label{eq:bounds}
 \sum_x(m_x-a_x)d_x\le\TV(m,r)\le\eta,\qquad
 \sum_xa_xe_x\le\sum_xr_xe_x\le\eta+\TV(m,r)\le2\eta.
\end{equation}
For the middle inequality, add and subtract $m_x\mathsf P(y\mid x)$ inside $r_x\mathsf P(y\mid x)-r_xR(y\mid x)$, sum absolute values over $(x,y)$, and divide by two. Substitution in \eqref{eq:alignment} gives $D-5\eta$. When $r=m$, the first loss in \eqref{eq:bounds} is zero and $\sum_xm_xe_x=\TV(\mathsf P,R)=\eta$, giving $D-2\eta$.
\end{proof}

The identity also specifies when the advantage reverses: the target's weighted conditional laws must be closer to one arm than the other. In particular, for $R_\lambda=(1-\lambda)\mathsf Q+\lambda\mathsf P$, $0\le\lambda\le1$, the gain is $(2\lambda-1)D$ and crosses zero at $\lambda=1/2$.

For the direct three-token laws of Theorem \ref{thm:triple}, set the autoregressive target to $R=\mathsf P^{(3)}$. Both proposal arms match the target's first two-token \emph{joint} prefix, so those two proposals are always accepted. Exact verification of the third then gives
\begin{equation}\label{eq:threeaccept}
 L(\mathsf P^{(3)},R)=3,\qquad
 L(\mathsf Q^{(3)},R)=3-4|\delta|d_1d_2d_3.
\end{equation}
Indeed, average the conditional third-token overlap over their common first-two-token prefix; its deficit is $\TV(\mathsf P^{(3)},\mathsf Q^{(3)})$. The favorable centers are constructed. No particular language model has been shown to lie in these target neighborhoods. A flexible classical proposal can reproduce either Bethe-derived token law exactly; the comparisons isolate the specified pairwise-moment-matched control.

\section{The XXZ branch and its correlation gap}

We instantiate the construction in a periodic chain with $L=8$, $M=3$, exchange coefficient $J=1$, and variable anisotropy $\Delta$; the first numerical instance uses $\Delta=1.5$:
\begin{equation}\label{eq:H}
 H(\Delta)=-\sum_{x=0}^{7}\left(S_x^xS_{x+1}^x+S_x^yS_{x+1}^y+
 \Delta S_x^zS_{x+1}^z\right),\qquad S_{8}^{a}=S_0^{a}.
\end{equation}
On the 56 occupation configurations, the diagonal entry at $z$ is $-\Delta\sum_x(z_x-\tfrac12)(z_{x+1}-\tfrac12)$; each allowed nearest-neighbor exchange has off-diagonal entry $-\tfrac12$. At $\Delta=1.5$, diagonalizing the seven-dimensional translation-invariant sector gives a normalized eigenvector at energy $-1.3898755836515198$. This state also has an explicit coordinate Bethe solution. Set $k_1=-\kappa$, $k_2=0$, $k_3=\kappa$, where $\kappa=0.3333918699484193$. The last digits of these two values are refined with 75-digit arithmetic using the orbit matrix \eqref{eq:orbitH} and $E=\Delta-1-2\cos\kappa$. Define the two-body scattering factor
\begin{equation}\label{eq:S}
 S(k,q)=-\frac{1+e^{i(k+q)}-2\Delta e^{ik}}
                 {1+e^{i(k+q)}-2\Delta e^{iq}}.
\end{equation}
The periodic Bethe equations $e^{8ik_a}=\prod_{b\ne a}S(k_a,k_b)$ have maximum complex residual below $3\times10^{-15}$ in double-precision evaluation. For $x_1<x_2<x_3$, construct the coordinate wavefunction
\begin{equation}\label{eq:BA}
 \psi_{\rm BA}(x_1,x_2,x_3)=\sum_{\pi\in\mathfrak S_3}A_\pi
   e^{i\sum_{a=1}^3 k_{\pi(a)}x_a},\qquad
 A_\pi=\prod_{\substack{a<b\\\pi^{-1}(a)>\pi^{-1}(b)}}S(k_b,k_a).
\end{equation}
After normalization, $|\psi_{\rm BA}|^2$ agrees with the diagonalized eigenvector's squared amplitudes to less than $10^{-15}$ in maximum absolute configuration error. Its probabilities define $P$. At $\Delta=1.5$ these are real scattering momenta, not a complex bound-state string. A translation-invariant maximum-entropy law $Q$ is fitted on the same 56 configurations using pair counts at four separations. Their sum is fixed by $M=3$, so only \emph{three} couplings are independent. All 28 pair moments are checked individually.

The Hamiltonian eigen-residual for this state is $1.59\times10^{-15}$; the maximum pair-moment error for this comparator is $4.1\times10^{-16}$. For $A=Z_0$ and $B=Z_1Z_2$, the numerical tables are:
\begin{center}
\begin{tabular}{ccrr}
\toprule
$A$ & $B$ & XXZ-derived $P$ & matched pairwise $Q$\\
\midrule
0&0&0.5306538300&0.5173809665\\
0&1&0.0943461700&0.1076190335\\
1&0&0.3277643839&0.3410372474\\
1&1&0.0472356161&0.0339627526\\
\bottomrule
\end{tabular}
\end{center}
Both arms have $\Pr(A=1)=0.375$ and $\Pr(B=1)\approx0.1415817861$. The triple-moment gap is $\delta\approx0.0132728634970$; with deterministic binary channels $d_1=d_2=1$, the two-token total variation distance is $D\approx0.0265457269941$. For the target $R=\mathsf P$, Corollary \ref{cor:accept} gives expected accepted lengths $2$ and approximately $1.973454273006$, respectively. An unrestricted sufficient target radius is $D/5\approx0.005309145399$; within the fixed-first-marginal family the stronger radius is $D/2=|\delta|$. For the direct $(Z_0,Z_1,Z_2)$ readout, all three one-token and all three pair-token marginals match, while $\TV(P^{(3)},Q^{(3)})\approx0.0530914539881$. With the constructed target $R=P^{(3)}$, accepted lengths are $3$ and approximately $2.946908546012$. These are accepted \emph{proposed tokens per block}, not throughput. The decimal values have the reported numerical tolerances.

\subsection{Why the adjacent triple is complete here}

The XXZ probability law and its pairwise projection are invariant under rotations and reflections of the ring. Reflection commutes with $H(\Delta)$ and preserves the translation-invariant sector, since it conjugates one-site translation to its inverse. At each tabulated $\Delta$ and at $\Delta_*$, the tracked state is the only one at its energy in that sector, so $|\psi|^2$ is reflection-invariant. The maximum-entropy projection inherits this symmetry because its constraints do. The 56 configurations form five dihedral orbits. In the order represented by occupied sites $(012),(013),(014),(024),(025)$, their sizes are $(8,16,16,8,8)$. With spins $s_x=2z_x-1$, let $F_d(z)=\sum_{x=0}^{7}s_xs_{x+d\,({\rm mod}\,8)}$ for $d=1,2,3$. Each $F_d$ equals four times the occupied-pair count at separation $d$, minus four; the separation-4 pair count is fixed by these counts and $M=3$. The five orbit representatives have $(F_1,F_2,F_3)$ equal to
\begin{equation}\label{eq:orbits}
 (4,0,-4),\quad (0,0,0),\quad(0,-4,0),\quad
 (-4,4,-4),\quad(-4,0,4).
\end{equation}
The $4\times5$ matrix whose rows are $1,F_1,F_2,F_3$ has rank four and null vector $(1,-2,0,0,1)$. Thus for any two dihedrally invariant laws with equal one- and two-site moments, their \emph{orbit mass} difference is a scalar multiple of this vector. The event $Z_0Z_1Z_2=1$ occurs with probability $1/8$ within the first orbit and zero within the other four. Consequently,
\begin{equation}\label{eq:orbitgap}
 m^P-m^Q=\delta(8,-16,0,0,8),\qquad
 \TV(P,Q)=16|\delta|.
\end{equation}
At $\Delta=1.5$, the full 56-state gap is therefore $0.21236581595$. Among the 56 choices of site triples, 32 have moment gap $+\delta$ or $-\delta$ and 24 have zero gap. We choose $(0,1,2)$ because it is a contiguous, nearest-neighbor cluster and its moment is a complete coordinate for the matched-law difference under this symmetry. It is not a uniquely favorable site choice.

\subsection{Scattering dependence and an exact cancellation point}

We track the same translation-invariant eigenstate as $\Delta$ varies and reconstruct it from $k=(-\kappa,0,\kappa)$ and the scattering factor \eqref{eq:S}. Its eigenenergy obeys $E(\Delta)=\Delta-1-2\cos\kappa$. The nonzero-momentum real branch reaches $\kappa=0$ at $\Delta=8/5$: differentiating the two sides of $e^{8i\kappa}=S(\kappa,0)S(\kappa,-\kappa)$ at zero gives $8i$ and $i3\Delta/(\Delta-1)$, respectively. Across this coalescence, $\kappa$ becomes imaginary. The coordinate formula is singular at the exact coalescence, so that point is omitted from the numerical table.

\begin{center}
\begin{tabular}{crrr}
\toprule
$\Delta$ & $\kappa$ & $\delta(\Delta)$ & $2|\delta(\Delta)|$\\
\midrule
1.10&$0.805762$&$\phantom{-}0.01401556$&0.02803112\\
1.30&$0.600797$&$\phantom{-}0.01452557$&0.02905114\\
1.50&$0.333392$&$\phantom{-}0.01327286$&0.02654573\\
1.80&$-0.448597i$&$\phantom{-}0.00840462$&0.01680924\\
2.25&$-0.772132i$&$\phantom{-}0.00133555$&0.00267110\\
2.50&$-0.893929i$&$-0.00133120$&0.00266240\\
\bottomrule
\end{tabular}
\end{center}
For the listed points, the reconstructed Born probabilities agree with the diagonalized eigenstates to $6\times10^{-15}$ per configuration, all pair moments of each fitted control agree to $5\times10^{-15}$, and the absolute Bethe-equation residual is below $10^{-9}$. The gap is nonmonotone; the scattering solution alone does not imply that stronger anisotropy increases it.

Because $\delta$ changes sign between $\Delta=2.25$ and $\Delta=2.5$, continuity of the tracked branch forces a zero on that interval; the additional result is its exact location and an explicit pairwise form. At $\Delta_*=(3+\sqrt3)/2$, the five orbit-constant eigenvector amplitudes are proportional to $(-1,1,\sqrt3,1,1)$ and have eigenenergy $-(1+\sqrt3)/2$. Their squared, normalized \emph{per-configuration} probabilities are $(1,1,3,1,1)/88$ on the five orbits. Using the features in \eqref{eq:orbits}, this is exactly a pairwise exponential law with coefficients
\begin{equation}\label{eq:star}
 (\theta_1,\theta_2,\theta_3)=
 \left(-\frac{\log3}{8},-\frac{\log3}{4},-\frac{\log3}{8}\right).
\end{equation}
For clarity, in the ordered orbit-constant amplitude basis the restriction of $H(\Delta)$ is
\begin{equation}\label{eq:orbitH}
 \begin{pmatrix}
 -\Delta&-1&0&0&0\\
 -1/2&-1/2&-1/2&-1/2&0\\
 0&-1/2&-1/2&-1/2&-1/2\\
 0&-1&-1&\Delta&-1\\
 0&0&-1&-1&\Delta-1
 \end{pmatrix}.
\end{equation}
Although the orbit basis is not orthonormal, this matrix acts on the five per-configuration amplitudes. At $\Delta_*$, multiplying \eqref{eq:orbitH} by $(-1,1,\sqrt3,1,1)^\mathsf T$ yields $-(1+\sqrt3)/2$ times that vector. Substituting \eqref{eq:star} verifies its pairwise representation on every orbit. Hence $P_{\Delta_*}=Q_{\Delta_*}$, $\delta(\Delta_*)=0$, and this particular Bethe-derived token advantage vanishes exactly. Along the tracked branch the higher-order gap varies, and it disappears at this exact anisotropy under the matched pairwise projection.

The numerical calculations and validation checks in this section and for Theorems \ref{thm:triple}--\ref{thm:alignment} can be reproduced from the accompanying ancillary code and data files; see the ancillary \texttt{README.md} for run instructions.

\section{Interpretation and research boundary}

The positive result is a constructive bridge: a coordinate Bethe state supplies a parameter-dependent higher-order classical law; the gap varies along the XXZ Bethe branch, is complete under the sector's dihedral symmetry, and vanishes at an exact anisotropy. A nonlinear readout transfers that gap to an early two-token dependence with fixed marginals; a direct readout carries it to three tokens while retaining all pair-token marginals. Exact speculative verification gives a strict difference for constructed targets and, in the two-token case, an unrestricted neighborhood. The state at $\Delta=1.5$ has real scattering momenta; complex momenta arise further along the tracked branch. Neither regime is a physical model of language.

This construction is not a statement that text is a one-dimensional physical spin chain, nor that rejection sampling is a projective quantum measurement. Its representational result alone does not imply faster decoding: acceptance depends on alignment with the actual target conditionals, while serving speed also depends on proposal cost. A separate language-model study with strong classical controls and measured target-verified acceptance would be needed for a performance claim. Prior Born-machine sequence models and trainable token measurements provide broader modeling context \cite{miller,hou}; our result isolates this particular XXZ-derived, pairwise-moment-controlled causal readout.


{\small
\begin{thebibliography}{9}
\bibitem{bethe} H. Bethe. Zur Theorie der Metalle I: Eigenwerte und Eigenfunktionen der linearen Atomkette. \emph{Zeitschrift f\"ur Physik} 71, 205--226 (1931).
\bibitem{leviathan} Y. Leviathan, M. Kalman, and Y. Matias. Fast Inference from Transformers via Speculative Decoding. \emph{ICML}, PMLR 202:19274--19286 (2023). \url{https://proceedings.mlr.press/v202/leviathan23a.html}
\bibitem{ptp} F. Draxler, J. Will, F. M. Sofian, T. Karaletsos, S. Singh, and S. Mandt. Parallel Token Prediction for Language Models. \emph{ICLR} (2026). \url{https://openreview.net/pdf?id=AGJomYSrUG}
\bibitem{roudi} Y. Roudi, S. Nirenberg, and P. E. Latham. Pairwise Maximum Entropy Models for Studying Large Biological Systems: When They Can Work and When They Can't. \emph{PLoS Computational Biology} 5(5):e1000380 (2009). \url{https://doi.org/10.1371/journal.pcbi.1000380}
\bibitem{miller} J. Miller, G. Rabusseau, and J. Terilla. Tensor Networks for Probabilistic Sequence Modeling. \emph{AISTATS}, PMLR 130:3079--3087 (2021). \url{https://proceedings.mlr.press/v130/miller21a.html}
\bibitem{hou} W. Hou, M. Li, and Y.-Z. You. Sequential learning on a Tensor Network Born machine with Trainable Token Embedding. arXiv:2311.05050v2 (2025). \url{https://arxiv.org/abs/2311.05050}
\end{thebibliography}
}
\end{document}